%% file: main.tex
\documentclass[sigconf,nonacm]{acmart}

\usepackage{hyperref}
\usepackage{url}

\usepackage{graphicx}

\usepackage{booktabs}

\usepackage{amsmath}
\usepackage{amsthm}
\usepackage[table]{xcolor}
\usepackage{multirow}

\newtheorem{theorem}{Theorem}
\newtheorem{lemma}[theorem]{Lemma}

\newtheorem{remark}{Remark}

\usepackage{pifont}
\usepackage{algorithm}
\usepackage{algorithmic}
\usepackage{wrapfig}
\AtBeginDocument{%
  }

\setcopyright{none}

\copyrightyear{2027}
\acmYear{2027}
\acmDOI{}
\acmConference[]{}{}{}
\acmISBN{}

\begin{document}

\title{Asynchronous Sharpness-Aware Minimization for Fast and Accurate Deep Learning}
\renewcommand{\shorttitle}{Asynchronous Sharpness-Aware Minimization for Fast and Accurate Deep Learning}

\author{Junhyuk Jo}
\email{911whwnsgur@inha.edu}
\affiliation{%
  \institution{Inha University}
  \country{Republic of Korea}
}

\author{Jihyun Lim}
\email{wlguslim@inha.edu}
\affiliation{%
  \institution{Inha University}
  \country{Republic of Korea}
}

\author{Vincent-Daniel Yun}
\email{yunjuyou@usc.edu}
\affiliation{%
    \institution{University of Southern California}
  \country{USA}
}

\author{Sunwoo Lee}
\authornote{Corresponding author.}
\email{sunwool@inha.ac.kr}
\affiliation{%
  \institution{Inha University}
  \country{Republic of Korea}
}
\renewcommand{\shortauthors}{Jo et al.}

\begin{abstract}
Sharpness-Aware Minimization (SAM) is an optimization method that improves the generalization performance of machine learning models.
Despite its effectiveness, it has not been widely adopted in real-world applications due to its high computational cost.
Existing approaches reduce this overhead by approximating or reusing gradients in the perturbation step, but still assume sequential perturbation and update steps.
In this study, we introduce \texttt{AsyncSAM}, a variant of SAM that enables concurrent model perturbation and parameter updates across multiple compute resources.
By adjusting the perturbation batch size in a system-aware manner, our asynchronous scheme fully overlaps perturbation with updates, eliminating the overhead and allowing even CPUs to perform perturbations without extra time.
We also provide theoretical analysis showing that \texttt{AsyncSAM} achieves the same order of convergence as conventional SAM.
Extensive experiments demonstrate that \texttt{AsyncSAM} delivers superior generalization performance and the highest training throughput among state-of-the-art SAM variants.
For example, in ViT fine-tuning on CIFAR-100, it achieves accuracy comparable to SAM while maintaining training time nearly identical to SGD.
\end{abstract}




\keywords{Sharpness-Aware Minimization, Asynchronous Model Perturbation, Generalization Performance}

\maketitle

\input{section/1_intro}

\input{section/2_related}

\input{section/3_method}

\input{section/4_eval}

\input{section/5_conclusion}

\bibliographystyle{ACM-Reference-Format}
\bibliography{mybib}

\input{section/appendix}

\end{document}

%% file: section/1_intro.tex
\section {Introduction}
Sharpness-Aware Minimization (SAM) is an optimization method that improves generalization performance  of machine learning models \cite{foret2020sharpness}.
Although SAM improves generalization, its high computational cost has limited its adoption in real-world applications.
Specifically, it calculates gradients twice for a single training step -- one for model perturbation (gradient ascent) and another for model update (gradient descent).
To make SAM a more compelling option for deep learning applications, thus, addressing this high computational cost is crucial.


While several previous works have studied ways to mitigate the high computational cost of model perturbation in SAM, a recent trend is to partially discard perturbation gradients.
For example, LookSAM~\cite{liu2022towards} periodically computes perturbation gradients and reuses them across multiple iterations.
AE-SAM~\cite{jiang2023adaptive} alternates between SGD and SAM depending on the loss landscape geometry.
ESAM~\cite{du2021efficient} perturbs only a randomly selected subset of parameters.
Although these methods effectively reduce the computational cost, they inevitably lead to a non-negligible drop in accuracy due to the poor quality of the perturbation directions.

\begin{figure}[t]
\centering
\includegraphics[width=1\columnwidth]{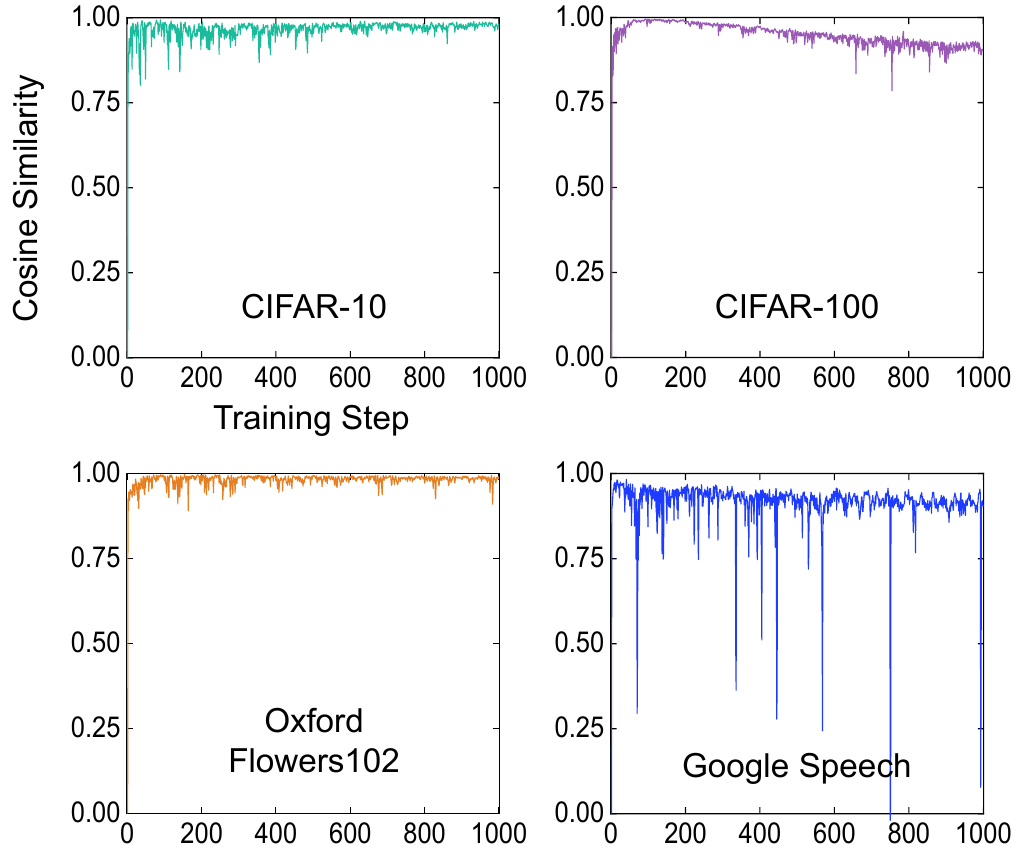}
\caption{
    The cosine similarity between gradients from two consecutive SGD iterations, measured over 1,000 iterations. All curves consistently maintain a high similarity ($> 0.8$), indicating that the parameter space with respect to a given dataset does not change drastically during training.
}
\label{fig:cossim}
\end{figure}

In this paper, we design a novel asynchronous model perturbation scheme which leverages a critical observation on gradient stability during neural network training.
Figure~\ref{fig:cossim} presents the cosine similarity between gradients from two consecutive SGD iterations, measured across several image classification benchmarks.
The similarity remains consistently high (greater than 0.8), indicating that the gradient direction in parameter space tends to be stable over short training intervals.
Similar observations have been reported in different contexts~\cite{azam2021recycling,liu2022towards}.
This observation naturally motivates us to develop an asynchronous-parallel model perturbation scheme.

\begin{table*}[t]
\small
\centering
\caption{A feature-wise comparison between \texttt{AsyncSAM} and other SAM-based methods.}
\begin{tabular}{lcccccccc}
\toprule
\multirow{2}{*}{Method} & Small Batch & No Variance & Model & Unbiased & Efficient & Input Data- & Overlap Perturbation & System-aware \\
& Availability & Tracking & Agnostic & Gradient & Perturbation & Independent & with Update & Batch Size \\ \midrule
SAM & \ding{51} & \ding{51} & \ding{51} & \ding{51} & \ding{55} & \ding{51} & \ding{55} & \ding{55} \\
Generalized SAM & \ding{51} & \ding{51} & \ding{51} & \ding{51} & \ding{55} & \ding{51} & \ding{55} & \ding{55} \\
SAM-ON & \ding{51} & \ding{51} & \ding{55} & \ding{51} & \ding{51} & \ding{51} & \ding{55} & \ding{55} \\
ESAM & \ding{51} & \ding{51} & \ding{51} & \ding{55} & \ding{51} & \ding{55} & \ding{55} & \ding{55} \\
LookSAM & \ding{55} & \ding{51} & \ding{51} & \ding{51} & \ding{51} & \ding{51} & \ding{55} & \ding{55} \\
MESA & \ding{51} & \ding{51} & \ding{51} & \ding{51} & \ding{51} & \ding{55} & \ding{55} & \ding{55} \\
AE-SAM & \ding{51} & \ding{55} & \ding{51} & \ding{51} & \ding{51} & \ding{55} & \ding{55} & \ding{55} \\
SAMPa & \ding{51} & \ding{51} & \ding{51} & \ding{51} & \ding{51} & \ding{51} & \ding{51} & \ding{55} \\
\rowcolor{lightgray!30} \textbf{AsyncSAM}& \ding{51} & \ding{51} & \ding{51} & \ding{51} & \ding{51} & \ding{51} & \ding{51} & \ding{51} \\
\bottomrule
\end{tabular}

\label{tab:summary}
\end{table*}

We propose \texttt{AsyncSAM}, an Asynchronous model perturbation-based Sharpness-Aware Minimization algorithm.
\texttt{AsyncSAM} decouples the model perturbation and update steps and hides the perturbation time with gradient descent time in SAM.
That is, we introduce a controlled degree of asynchrony so that the perturbation and update steps can run concurrently.
Instead of computing the perturbation gradient using the most recent model parameters, we employ slightly stale parameters to break the data dependency.
Our empirical results show that, with a system-aware choice of batch size, asynchronous perturbation can be performed even on relatively slow CPU cores, effectively eliminating the perturbation overhead.

Our study provides two key insights as follows.
First, gradients in large-scale deep learning tend to remain stable throughout training. As a result, a limited degree of staleness in model perturbation does not considerably harm the quality of the perturbation direction in parameter space, which in turn preserves the gradient norm penalization effect of SAM.
Second, although reducing the batch size for gradient ascent introduces stronger noise into the model perturbation, the resulting penalty is offset by the shorter training time made possible by asynchronous–parallel perturbation.
Based on these findings, we argue that SAM can be effectively applied to many real-world deep learning applications without significantly increasing training time.
Table~\ref{tab:summary} compares \texttt{AsyncSAM} with SAM variants.

We evaluate \texttt{AsyncSAM} on diverse deep learning benchmarks, including CIFAR-10, CIFAR-100, Google Speech Commands, Oxford\_Flowers102, Tiny-ImageNet, ImageNet and Vision Transformer fine-tuning.
We also compare it with state-of-the-art SAM variants designed to reduce perturbation cost.
Our experiments show that \texttt{AsyncSAM} eliminates perturbation overhead while maintaining superior generalization.
For example, it achieves $92.66\%$ accuracy on CIFAR-10, comparable to SAM’s $92.53\%$, with training throughput similar to SGD.
It also matches SAM in Vision Transformer fine-tuning without increasing training time.
These results suggest that \texttt{AsyncSAM} is a practical and efficient optimizer for real-world deep learning tasks.

%% file: section/2_related.tex
\section {Related Work}
\textbf{Problem Statement} --
We consider the following standard minimization problem:
\begin{align}
    L(w) := \mathbb{E}_{\xi \sim D} \left[ l(w, \xi) \right],
\end{align}
where $\xi$ is a data sample randomly drawn from the distribution $D$, and $l(\cdot)$ is a non-convex and smooth objective function.
For convenience, we abbreviate $\mathbb{E}_{\xi \sim D}[\cdot]$ as $\mathbb{E}[\cdot]$.

\textbf{Sharpness-Aware Minimization (SAM)} -- SAM is an optimization method that improves the model's generalization performance~\cite{foret2020sharpness}.
It first computes the gradient of the loss function and uses it to perturb the model.
Next, the gradient is computed again with the perturbed model parameters.
Finally, the model is reverted to its original parameters and updated using the gradient calculated from the perturbed model.
The parameter update rule is as follows.
\begin{align}
    w_{t+1} = w_{t} - \eta \nabla L\left( w_{t} + r \frac{\nabla L(w_t)}{\| \nabla L(w_t) \|} \right),
\label{eq:sam}
\end{align}
where $w_t$ is the model parameters at iteration $t$, $\eta$ is the learning rate, and the $r$ is a hyper-parameter that determines how strongly the model is perturbed.
Recently, it has been shown that SAM improves generalization performance by penalizing the gradient norm~\cite{zhao2022penalizing}.

\textbf{System-Efficient Variants of SAM} -- There are several previous works that tackle the expensive model perturbation cost in SAM.
\textit{LookSAM} reuses gradients for multiple perturbations~\cite{liu2022towards}, but its generalization deteriorates with smaller batch sizes.
While \textit{AE-SAM} reduces the overall training time by adaptively alternating between SGD and SAM based on the loss landscape geometry~\cite{jiang2023adaptive}, it does not fundamentally reduce the perturbation cost.
\textit{ESAM} perturbs a random subset of parameters~\cite{du2021efficient}, which may introduce bias due to selective data updates.
\textit{SAMPa} perturbs and updates the model simultaneously using an auxiliary sequence~\cite{xie2024sampa}.
\textit{SAM-ON} perturbs only batch normalization layers~\cite{mueller2024normalization}, limiting its applicability to architectures without them.
\textit{MESA} employs trajectory-based gradient approximation~\cite{du2022sharpness}, whose effectiveness depends on data complexity.
A layer-wise SAM has been also proposed, which perturbs only a subset of network layers making a trade-off between system efficiency and generalization performance~\cite{lee2024layer}.
SAM has also been extended to federated learning, showing that local perturbation can improve global generalization~\cite{qu2022generalized}.
While these methods reduce the perturbation overhead, they do not completely eliminate it, still requiring a non-negligible amount of additional time.

\textbf{Neural Network Training on Heterogeneous Systems} --
Several recent studies have explored opportunities for utilizing heterogeneous systems in neural network training.
\textit{DistReal}~\cite{rapp2022distreal} is a system-aware federated learning method designed for heterogeneous environments.
\textit{EmbracingFL}~\cite{lee2024embracing} proposes a partial model training strategy for heterogeneous federated learning.
\textit{InclusiveFL}~\cite{liu2022no} is a federated learning framework that adapts the model width based on device resource capabilities.
However, these works primarily focus on distributed learning settings with resource-constrained edge devices.
To the best of our knowledge, relatively few studies have investigated efficient neural network training that exploits in-node heterogeneous resources.
Zheng et al.~\cite{zheng2022distributed} proposed a distributed training method for hybrid (CPU--GPU) systems.
While this work considers in-node heterogeneous resources, it focuses specifically on graph neural networks.
Li et al.~\cite{li2023cotrain} studied how to leverage the large memory capacity of CPUs to support neural network training on GPU servers.
Their work focuses on scheduling CPU and GPU workloads to maximize parallelism.

%% file: section/3_method.tex
\section {Method}
In this section, we first define an asynchronous-parallel model perturbation method, and then propose \texttt{AsyncSAM}, a general SAM optimizer built upon this method.
Finally, we provide a theoretical analysis of its convergence properties to demonstrate its efficacy.

\subsection {Asynchronous-Parallel Model Perturbation with Fixed Staleness}
The first step of SAM is to perturb the model parameters via gradient ascent \cite{foret2020sharpness}.
After the perturbation, the gradient is computed, and the model is finally updated using this gradient.
Our key idea is to introduce a fixed degree of staleness into the model perturbation, enabling the perturbation and update steps to be performed concurrently.
This controlled staleness breaks the data dependency between the two steps, allowing the perturbation to be executed independently in a separate process.
The asynchronous model perturbation can be defined as follows.
\begin{align}
    w_{t+1} = w_{t} -\eta \nabla L_{t} \left( w_{t} + r \frac{\nabla L_{t-\tau}(w_{t-\tau})}{ \| \nabla L_{t-\tau} (w_{t-\tau}) \|} \right). \label{eq:async1}
\end{align}
The main difference between (\ref{eq:async1}) and the parameter update rule of the original SAM lies in the asynchrony term, $\tau$, on the right-hand side.
When $\tau = 0$, the method reduces to the original SAM, which updates the model only after perturbation.
A positive $\tau$ indicates that the model is perturbed using slightly outdated gradients.
Ideally, by employing separate processes for the asynchronous perturbation and update steps, the overall training time can match that of SGD, without incurring additional overhead for model perturbation.

\textbf{Discussion on Gradient Staleness} --
While asynchronous model perturbation (\ref{eq:async1}) improves parallelism by removing data dependencies, it may slow down convergence if the gradient ascent step becomes too inaccurate.
Specifically, when $\tau$ is large, the stale gradient $\nabla L_{t-\tau}(w_{t-\tau})$ can deviate significantly from the current gradient $\nabla L_t(w_t)$, thereby reducing the effectiveness of gradient norm penalization.
Therefore, $\tau$ should be kept sufficiently small to ensure that the stale gradient remains close to the current one.
Azam et al., investigated how long stochastic gradients remain similar across SGD iterations~\cite{azam2021recycling}.
Their results suggest that first-order gradients remain similar over several iterations, supporting the effectiveness of our approach.

\textbf{Comparison to Representative SOTA Methods} --
LookSAM~\cite{liu2022towards} exploits the stability of stochastic gradients by reusing a gradient for $k$ iterations, where $k$ is a hyperparameter.
This stability, however, depends on the training data and the learning rate, making it difficult to determine an appropriate $k$ in advance.
For instance, with a large learning rate, the gradients can change substantially due to rapid movement in the loss landscape.
Consequently, reducing the degree of gradient staleness increases the likelihood of maintaining stable gradient dynamics, thereby improving the effectiveness of gradient-norm penalization.
\texttt{AsyncSAM} introduces a similar hyperparameter $\tau$, however, we find that $\tau=1$ effectively hides the perturbation overhead while minimizing staleness.

SAMPa~\cite{xie2024sampa} also parallelizes the gradient computations for perturbation and update. 
However, it computes the perturbation gradient from an auxiliary sequence and assumes a homogeneous system with two similarly capable compute devices.
In contrast, \texttt{AsyncSAM} perturbs the model using the gradient computed at the previous model parameters, allowing the perturbation and update steps to overlap.
In addition, our study explores how to adjust the perturbation batch size in a system-aware manner, making a practical trade-off between system efficiency and generalization improvement.


\subsection {Asynchronous-Parallel Model Perturbation on Heterogeneous Systems} \label{sec:batch}
In modern systems for machine learning, it is common to use accelerators such as GPUs to train machine learning models.
Without loss of generality, we consider a system with at least two compute resources of possibly different capabilities (e.g., a CPU and a GPU).
To allow the SAM algorithm to fully exploit such heterogeneous resources while preserving its convergence properties, we propose assigning gradient descent to the faster resource and gradient ascent to the slower one.
Under this general setting, the parameter update rule of SAM can be written as follows.
\begin{align}
    w_{t+1} &= w_{t} -\eta \nabla L_{t}^b \left( w_{t} + r \frac{\nabla L_{t-1}^{b'}(w_{t-1})}{ \| \nabla L_{t-1}^{b'} (w_{t-1}) \|} \right), b' \leq b, \label{eq:async2}
\end{align}
where $\nabla L_{t}^{b}(w_{t})$ indicates the mini-batch gradient computed from $b$ random training samples at iteration $t$ and $\nabla L_{t-1}^{b'}(w_{t-1})$ is the mini-batch gradient computed from $b'$ random samples at the previous iteration.
\begin{figure*}[t]
\centering
\includegraphics[width=1\textwidth]{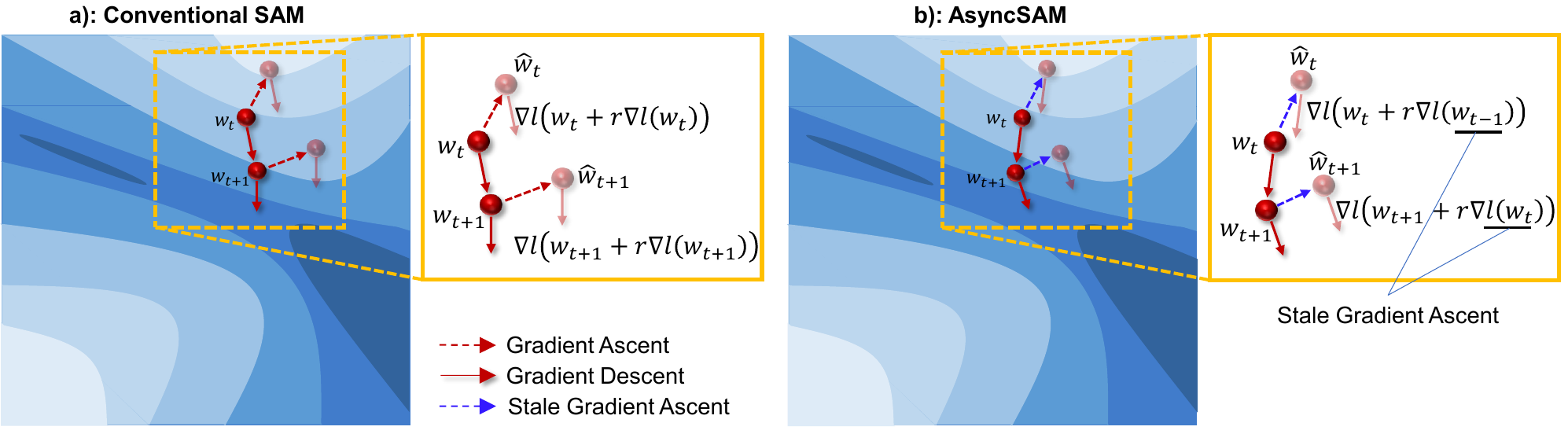}
\caption{
    \textbf{a}) A schematic illustration of conventional SAM. The model is updated twice using Eq.~\ref{eq:sam}.
    The model $w_t$ and $w_{t+1}$ are perturbed using the lastest gradients, $\nabla l(w_t)$ and $\nabla l(w_{t+1})$, respectively.
    \textbf{b}) A schematic illustration of the proposed \texttt{AsyncSAM}. Instead of using the latest gradient for the model perturbation, one-iteration-stale gradients, $\nabla l(w_{t-1})$ and $\nabla l(w_t)$ are used to perturb $w_t$ and $w_{t+1}$, respectively.
    When the staleness $\tau$ is sufficiently small, the stale gradients are close to the latest gradients, and the model is expected to move toward the same flat minimum.
}
\label{fig:schematic}
\end{figure*}

The batch size $b'$ directly determines the time required for the gradient ascent.
In the original SAM, $b'$ is assumed to be equal to $b$.
We propose adjusting $b'$ in a system-aware manner such that $b' = (T_f / T_s ) \times b$.
Here, $T_s$ is the single data gradient calculation time on the slow compute resource and $T_f$ is that on the fast resource.
In this way, we can automatically determine the maximum batch size $b'$ that equalizes the gradient ascent and descent times.
With a dedicated process, the gradient ascent computation can be fully overlapped with the gradient descent, effectively hiding the model perturbation time.
Note that $b'$ decreases as the performance gap between the two resources widens.

In (\ref{eq:async2}), we intentionally fix $\tau=1$.
One might think the $\tau$ could be increased instead of reducing $b'$ to match the gradient descent time and the ascent time.
Previous studies already have shown that they have the same complexity of the adverse impact on the convergence rate.
The degree of asynchrony $\tau$ has the impact of $\mathcal{O}(\tau)$ on the convergence rate \cite{koloskova2022sharper}.
In the meantime, it has also been known that $\textbf{var}(\nabla L^b) \propto b^{-1}$~\cite{gower2019sgd,smith2017bayesian}.
However, reducing the batch size is more beneficial in terms of system efficiency.
A smaller perturbation batch size $b'$ reduces both computation time and memory footprint, whereas greater asynchrony only increases the gradient descent time that can overlap with the ascent time.
Therefore, we fix $\tau = 1$ and adjust $b'$ accordingly.

\subsection {Asynchronous SAM Algorithm}

We construct \texttt{AsyncSAM}, an asynchronous-parallel model perturbation-based SAM, by combining the two aforementioned methods.
Figure~\ref{fig:schematic} illustrates the proposed method.
In Figure~\ref{fig:schematic}.a, two consecutive SAM iterations are depicted in the parameter space projected onto a two-dimensional plane.
The two gradient ascent steps perturb the model in nearly the same direction, consistent with our observation in Figure~\ref{fig:cossim}.
Figure \ref{fig:schematic}.b shows the schematic illustration of the proposed method.
Instead of the latest gradients (red arrows), our method uses $\tau = 1$ iteration-old gradients (blue arrows) for the model perturbation.
Although the gradients are staled, they are still similar to the latest gradients.
Thus, the model is expected to move toward a similar direction resulting in reaching the same flat minimum.

Algorithm~\ref{alg:asyncsam} presents the pseudocode of \texttt{AsyncSAM}, corresponding to Figure~\ref{fig:schematic}.b.
The algorithm assumes two threads (or processes) running concurrently.
Thread~$0$ computes $\nabla L_t^{b'}(w_t)$ at line~5, while thread~$1$ simultaneously computes $\nabla L_t^b(\hat{w}_t)$ using the perturbed model and updates the original parameters at lines~7–9.
Note that Algorithm~\ref{alg:asyncsam} fixes $\tau = 1$, which minimizes gradient ascent staleness and prevents the stochastic gradient noise from increasing excessively when a smaller batch size is used ($b' < b$).

\texttt{AsyncSAM} uses different mini-batches for perturbation and update.
This mini-batch mismatch does not reduce its effectiveness.
Under uniform random sampling, each mini-batch gradient serves as an unbiased estimator of the full gradient.
That is, $\mathbb{E}[\nabla L_{t}^b(w_t)] = \nabla L(w_t)$ and $\mathbb{E}[\nabla L_{t-1}^{b'}(w_{t-1})] = \nabla L(w_{t-1})$.
As shown in Figure~\ref{fig:cossim}, therefore, the gradients at consecutive iterates are highly aligned, so our method still benefits from gradient stability in expectation.

\begin{algorithm}[t]
\caption{
    \texttt{AsyncSAM} pesudo code.
}

\begin{algorithmic}[1]
    \STATE{\textbf{Input:} $r$: the ascent scaling factor, $T$: the number of training iterations, $b$: the gradient descent batch size, $b'$: the gradient ascent batch size, $\eta$: the learning rate}
    \STATE {Compute gradient: $g_0 = \nabla L_0^b(w_0)$.}
    \FOR{$t \in \{1, \cdots, T-1 \}$}
        \IF{Thread ID is 0}
            \STATE {Compute $\nabla L^{b'}_t(w_t)$.}
        \ELSE
            \STATE {Perturb the model: $\hat{w}_t = w_t + r \frac{\nabla L_{t-1}^{b'}(w_{t-1})}{\| \nabla L_{t-1}^{b'}(w_{t-1}) \|}$.}
            \STATE {Compute gradient: $g_t = \nabla L_t^b(\hat{w}_t)$.}
            \STATE {Update model: $w_{t+1} = w_t -\eta g_t$.}
        \ENDIF
    \ENDFOR
    \STATE{\textbf{Output:} $w_{T}$}
\end{algorithmic}
\label{alg:asyncsam}
\end{algorithm}

\subsection {Convergence Analysis}
We analyze the convergence behavior of \texttt{AsyncSAM}.
Our goal is to show that \texttt{AsyncSAM} guarantees a minimum level of performance, making it safely applicable to real-world applications.
The analysis assumes the followings.
\\
\textbf{Assumption 1.} (Smoothness): There exists a constant $\beta \geq 0$ such that $\| \nabla l(u) - \nabla l(v) \| \leq \beta \| u - v \| \mathrm{ , } \forall u, v \in \mathbb{R}^d$.
\\
\textbf{Assumption 2.} (Unbiased Gradient): The gradient estimator is unbiased such that $\mathbb{E}_{\xi \sim D}\left[ \nabla l(w, \xi) \right] = \nabla L(w)$.
\\
\textbf{Assumption 3.} (Bounded Variance): There exists a constant $\sigma \geq 0$ such that $\mathbb{E}_{\xi \sim D}[ \| \nabla l(w, \xi) - \nabla L(w) \|^2] \leq \sigma^2$.

The following lemma provides a bound on the loss reduction per update.
All proofs are given in the Appendix.

\noindent

\begin{lemma}
\label{lemma:async_frame}
Under Assumption 1--3, Asynchronous SAM guarantees the following if $\eta \leq \frac{1}{2\beta}$ and $r \leq \frac{1}{\beta}$:
\begin{align}
\mathbb{E}[L(w_{t+1})] &\leq \mathbb{E}[L(w_t)] -\frac{\eta}{4} \mathbb{E}\left[ \left\| \nabla L(w_t) \right\|^2 \right] \nonumber \\
&\qquad + \eta r\beta \mathbb{E}\left[ \left\| \nabla L(w_{t-\tau}) \right\|^2 \right] \nonumber \\
&\qquad + \eta\beta \left(\eta + \frac{r}{2} \right) \frac{\sigma^2}{b} + 2\eta^2 r \beta^2 \frac{\sigma^2}{b'}. \nonumber
\end{align}
\end{lemma}

Based on the above result, we analyze the convergence rate of Algorithm \ref{alg:asyncsam} as follows.

\noindent

\begin{theorem}
Under assumption 1--3, if $\eta \leq \frac{1}{2\beta}$ and $r < \frac{1}{4\beta}$, asynchronous SAM satisfies:
\begin{align}
\frac{1}{T} \sum_{t=0}^{T-1} \mathbb{E} \left[\left\| \nabla L(w_t) \right\|^2 \right] &\leq \frac{4}{\eta(1-4r\beta)T} \left( \mathbb{E}\left[L(w_0)\right] - \mathbb{E}\left[L(w_T) \right] \right) \nonumber \\
&\qquad + \frac{4r\beta \tau}{(1-4r\beta) T} \mathbb{E} \left[ \left\| \nabla L(w_0) \right\|^2 \right] \nonumber \\
&\qquad + \frac{4\beta(\eta + \frac{r}{2})}{1-4r\beta} \cdot \frac{\sigma^2}{b} + \frac{8\eta r\beta^2}{1-4r\beta} \cdot \frac{\sigma^2}{b'}.
\end{align}
\end{theorem}

\begin{remark}[Convergence guarantee]
For non-convex and smooth problems, Algorithm \ref{alg:asyncsam} guarantees a convergence regardless of the values of gradient ascent batch size $b'$.
The first and second terms on the right-hand side are $\mathcal{O}(\frac{1}{\eta T})$ and $\mathcal{O}(\frac{\tau}{T})$, respectively, and thus vanish as $T \rightarrow \infty$.
While the noise terms do not depend on $T$, they also vanish as $\eta \rightarrow 0$, which is standard under a learning rate decay schedule.
This convergence rate has the same complexity with that of the conventional SAM with stochastic model perturbation~\cite{andriushchenko2022towards}.
\end{remark}

\begin{remark}[Impact of asynchronous model perturbation]
The second constant term on the right-hand side grows linearly with the degree of asynchrony $\tau$.
However, this term vanishes as $T \rightarrow \infty$, it has no effect on the asymptotic convergence rate.
\end{remark}

\begin{remark}[Impact of perturbation batch size $b' \leq b$]
As $b'$ decreases, the variance term on the right-hand side increases making the training loss converge more slowly.
That is, if the compute resources used for model perturbation and model update differ significantly in computational power, the convergence will become slower.
The last term on the right-hand side denotes the penalty from asynchronous gradient ascent, where gradient staleness slows convergence.
Nonetheless, the proposed method hides the gradient ascent time behind the gradient descent, yielding a practical trade-off between slower convergence and shorter wall-clock time.
\end{remark}

%% file: section/4_eval.tex
\section {Performance Evaluation} \label{sec:eval}

\begin{table*}[ht]
\small
\centering
\caption{
    Classification performance comparisons among various SAM-based methods.
    This comparison shows that \texttt{AsyncSAM} achieves comparable accuracy to other SAM variants while achieving superior training throughput (Fig.~\ref{fig:tradeoff}).
}
\begin{tabular}{lccccccc} \toprule
\multirow{2}{*}{Algorithm} & CIFAR-10 & CIFAR-100 & Oxford\_Flowers102 & Google Speech & CIFAR-100 & Tiny-ImageNet & ImageNet \\
 &  (ResNet20) & (WRN-28) & (WRN-16) & (CNN) & (ViT Fine-Tuning) & (ResNet50)& (ResNet50)\\ \midrule
SGD & $91.96 \pm 0.3\%$ & $79.70 \pm 0.2\%$ & $61.37 \pm 2.2\%$ & $91.63 \pm 0.4\%$ & $89.87 \pm 0.1\%$ & $65.93 \pm 0.3\%$& $76.61 \pm 0.3\%$\\ \midrule
SAM & $92.53 \pm 0.4\%$ & $80.68 \pm 0.2\%$ & $60.82 \pm 0.6\%$ & $91.48 \pm 0.1\%$ & $92.54 \pm 0.1\%$ & $67.94 \pm 0.1\%$& $76.97 \pm 0.1\%$\\
Generalized SAM & $92.64 \pm 0.4\%$ & $81.19 \pm 0.2\%$ & $\mathbf{63.93} \pm 0.5\%$ & $92.03 \pm 0.2\%$ & $92.62 \pm 0.1\%$ & $68.31 \pm 0.1\%$ & $\mathbf{77.21} \pm 0.1\%$\\
ESAM & $92.20\pm0.2 \%$ & $78.74 \pm 0.1\%$ & $54.07 \pm 2.4\%$ &$ 86.99 \pm 0.5\%$& $91.94 \pm 0.1\%$ & $67.59 \pm 0.5\%$ & $76.73 \pm 0.1\%$\\
LookSAM & $91.62 \pm 0.3\%$ & $79.98 \pm 0.1\%$ & $44.17 \pm 1.7\%$ & $82.67 \pm 2.0\%$ & $91.90 \pm 0.1\%$ & $68.06 \pm 0.2\%$& $72.70 \pm 0.1\%$\\
MESA & $92.23 \pm 0.3\%$ & $80.15 \pm 0.1\%$ & $57.89 \pm 0.8\%$ &$91.23 \pm 0.1\%$ & $89.27 \pm 0.2\%$ & $67.62 \pm 0.2\%$& $76.88 \pm 0.2\%$\\
AE-SAM & $92.60 \pm 0.2\%$ & $80.38 \pm 0.1\%$ & $59.71 \pm 1.0\%$ & $91.50 \pm 0.1\%$ & $92.09 \pm 0.2\%$ & $67.22 \pm 0.1\%$& $76.75 \pm 0.2\%$\\
SAMPa & $92.52 \pm 0.1\%$ & $80.72 \pm 0.4\%$ & $63.07 \pm 2.5\%$ & $91.95 \pm 0.7\%$ & $92.71 \pm 0.1\%$ & $68.22 \pm 0.3\%$& $76.69 \pm 0.1\%$\\
\rowcolor{lightgray!30} \textbf{AsyncSAM} & $\mathbf{92.66} \pm 0.2\%$ & $\mathbf{81.52} \pm 0.4\%$ & $63.86 \pm 1.1\%$ & $\mathbf{92.08} \pm 0.2\%$ & $\mathbf{92.73} \pm 0.1\%$ & $\mathbf{68.52} \pm 0.1\%$ & $76.78 \pm 0.1 \%$\\ \bottomrule
\end{tabular}

\label{tab:classification}
\end{table*}

\subsection {Experimental Settings}
We evaluate \texttt{AsyncSAM} on representative benchmarks, including CIFAR-10 \cite{krizhevsky2009learning} (ResNet20 \cite{he2016deep}), CIFAR-100 (Wide-ResNet28 \cite{zagoruyko2016wide}), Oxford\_Flowers102 \cite{Nilsback08} (Wide-ResNet16), Google Speech Commands \cite{warden2018speech} (CNN), and ImageNet~\cite{deng2009imagenet} (ResNet50).
We also evaluate CIFAR-100 fine-tuning using an Vision Transformer (ViT) \cite{dosovitskiy2020image}.
Experimental software is developed in PyTorch 2.5.0.
\texttt{AsyncSAM} is implemented using MPI with two parallel processes, one for gradient ascent and the other for gradient descent.
See the Appendix for the detailed hyperparameter settings.
All results are averaged over at least three runs.

\subsection {Classification Performance}

\textbf{Accuracy Comparison} -- Table \ref{tab:classification} presents classification performance comparisons.
We compare \texttt{AsyncSAM} to extensive SOTA variants (Generalized SAM, ESAM, LookSAM, MESA, AE-SAM, and SAMPa) as well as the conventional mini-batch SGD and the original SAM~\cite{foret2020sharpness}.
Generalized SAM is a variant of SAM which updates the model using both the gradient ascent and descent~\cite{zhao2022penalizing}.
ESAM is a variant of SAM that takes advantage of the partial model perturbation and a biased gradient estimator~\cite{du2021efficient}.
LookSAM is another variant of SAM that reuses the gradient ascent~\cite{liu2022towards}.
MESA is a memory-efficient version of SAF (Sharpness-Aware Minimization for Free) that perturbs the model using a trajectory loss~\cite{du2022sharpness}.
AE-SAM runs SGD and SAM interchangeably based on the gradient norm~\cite{jiang2023adaptive}.
SAMPa perturbs the model using the previous model that is updated using SGD~\cite{xie2024sampa}.

Overall, \texttt{AsyncSAM} achieves significantly higher accuracy than SGD.
Across all benchmarks, its accuracy is comparable to that of Generalized SAM, which attains the best validation accuracy among all SOTA SAM variants.
Several SOTA methods perform notably worse than the original SAM or even SGD (e.g., MESA on CIFAR-100 and LookSAM on Oxford\_Flowers102 and Google Speech), whereas \texttt{AsyncSAM} consistently outperforms them by a large margin.
These results empirically demonstrate that the proposed asynchronous-parallel model perturbation preserves SAM’s superior generalization performance, even when the data dependency between model perturbation and model update is removed.
We also confirm that \texttt{AsyncSAM} is effective not only for moderately sized ResNet architectures but also for a large-scale Vision Transformer (ViT-b16).


\textbf{Accuracy--Throughput Trade-off} --Figure~\ref{fig:tradeoff} shows the accuracy--throughput trade-off on CIFAR-10 (ResNet20), where each point corresponds to the training throughput and final accuracy achieved by a given method ($b=128$, NVIDIA RTX A6000 GPU).
Methods located closer to the top-right corner achieve a better trade-off, attaining higher accuracy without sacrificing throughput.
As expected, the original SAM and the Generalized SAM achieve the highest accuracy among the baselines but incur similarly low throughput, differing only in the weighting factor.
LookSAM and ESAM offer only a marginal throughput gain over SAM while suffering a noticeable drop in accuracy, placing them in the lower-left region of the plot.
\texttt{AsyncSAM} occupies the top-right region, matching the accuracy of the Generalized SAM while running at a throughput comparable to SGD, which confirms that the asynchronous model perturbation improves system efficiency without compromising accuracy.
Among the baselines, MESA achieves the highest throughput by introducing an additional regularization term in the loss rather than a second gradient computation, but at the cost of lower accuracy than \texttt{AsyncSAM}.
SAMPa exploits parallelism similarly to \texttt{AsyncSAM} but performs an additional update step, resulting in slightly lower throughput and accuracy.

\begin{figure}[t]
\centering
\includegraphics[width=\columnwidth]{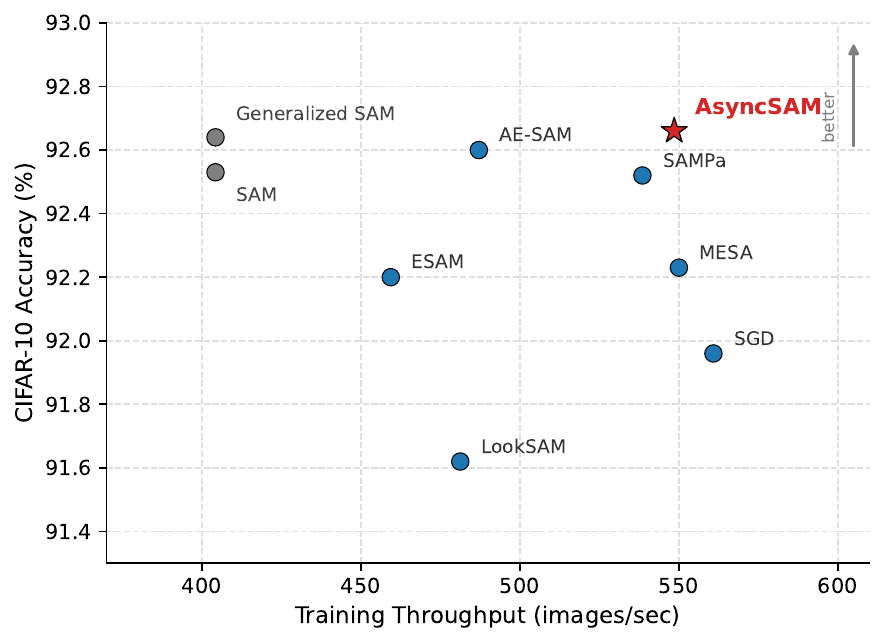}
\caption{
Accuracy–throughput trade-off on CIFAR-10 (ResNet20). \texttt{AsyncSAM} matches the accuracy of Generalized SAM while achieving throughput comparable to SGD. 
}
\label{fig:tradeoff}
\end{figure}

\begin{figure*}[t]
\centering
\includegraphics[width=1\textwidth]{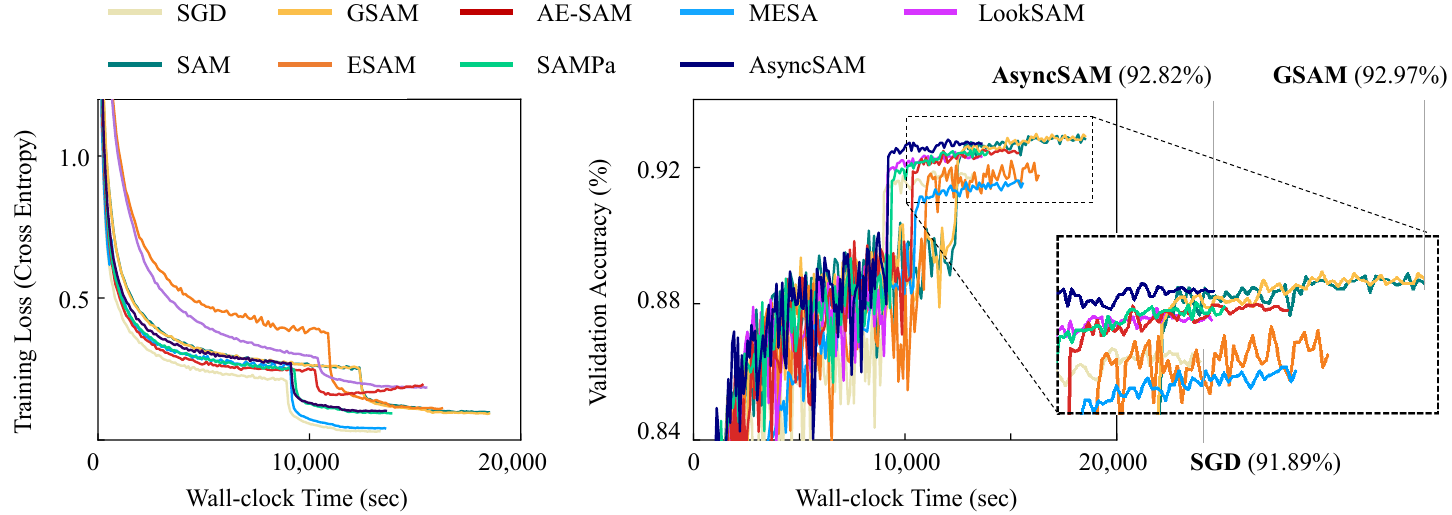}
\caption{
    CIFAR-10 accuracy-to-time comparison (the best curves achieved among three independent runs). \texttt{AsyncSAM} achieves comparable accuracy to the Generalized SAM~\cite{zhao2022penalizing} while taking a similar amount of time as SGD.
}
\label{fig:curves}
\end{figure*}

\textbf{Time-to-accuracy Analysis} --
The time-to-accuracy curves in Figure~\ref{fig:curves} clearly show the performance gains achieved by \texttt{AsyncSAM}.
The y-axis represents the training loss or validation accuracy, while the x-axis denotes the wall-clock time (in seconds).
The curves are obtained from training ResNet20 on CIFAR-10 using a single NVIDIA A6000 GPU, with \texttt{AsyncSAM} computing the perturbation gradient on an AMD EPYC 7452 CPU.
Consistent with the throughput results, Generalized SAM requires the longest time to complete the same number of training epochs.
With AE-SAM, SAM steps account for roughly half of the total updates, so it still takes significantly longer than SGD.
Finally, \texttt{AsyncSAM} achieves comparable accuracy to Generalized SAM while requiring nearly the same wall-clock time as SGD.

\begin{table*}[t]
\small
\centering
\caption{
    GLUE, a representative Natural Language Processing (NLP), benchmark results. \texttt{AsyncSAM} achieves the best fine-tuning performance in most of the tasks, resulting in the best average score.
}
\begin{tabular}{lccccccccc} \toprule
\multirow{2}{*}{Algorithm} & CoLA & SST-2 & MRPC & STS-B & QQP & MNLI & QNLI & RTE & All \\ 
& Mcc & Acc & Acc & Corr & Acc & m/mm & Acc & Acc & Avg \\ \midrule
SGD & $24.80 \pm 1.2$ & $90.45 \pm 0.1$ & $76.84 \pm 0.1$ & $88.26 \pm 0.1$ & $88.32\pm 0.1$ & $80.03 \pm 0.3$ & $84.12\pm 0.1$ & $53.32 \pm 1.1$ & 
$73.27$ \\
SAM & $25.11 \pm 0.9$ & $90.80 \pm 0.1$ & $84.07 \pm 0.2$ & $88.18 \pm 0.1$ & $88.23\pm 0.1$ & $80.15 \pm 0.1$ & $84.11\pm 0.6$& $54.88 \pm 2.7$ & $74.44$ \\
ESAM & $18.20\pm 0.1$ & $90.60\pm 0.1$ & $79.90\pm 0.7$ & $87.89\pm 0.1$ & $88.31\pm 0.1$ &  $\mathbf{80.58}\pm 0.1$& $83.76\pm 0.1$& $53.43\pm 0.7$ & $72.83$ \\
LookSAM & $25.76\pm 1.7$ & $90.42\pm 0.1$ & $82.23\pm 1.1$ & $86.95\pm 0.8$ & $88.30\pm 0.1$ & $80.47\pm 0.1$ & $\mathbf{84.82}\pm 0.1$& $\mathbf{55.23}\pm 0.7$ & $74.27$ \\
\rowcolor{lightgray!30} \textbf{AsyncSAM} & $\mathbf{26.80} \pm 1.1$ & $\mathbf{90.97} \pm 0.2$ & $\mathbf{84.56} \pm 0.2
$ & $\mathbf{88.18} \pm 0.2$ & $\mathbf{88.40}\pm 0.1$ & $80.46 \pm 0.1$ & $84.56\pm 0.1$& $54.69 \pm 1.0$ & $\mathbf{74.83}$ \\ \bottomrule
\end{tabular}

\label{tab:nlp}
\end{table*}

\textbf{NLP Fine-Tuning} --
We evaluate \texttt{AsyncSAM} using GLUE~\cite{wang2018glue}, a popular NLP fine-tuning benchmark.
We train T5~\cite{raffel2020exploring}, a widely-used transformer, on each dataset and report both the individual and average accuracy.
Table~\ref{tab:nlp} presents the results.
The accuracy gap between SGD and SAM is smaller in fine-tuning than in training from scratch, consistent with the trend reported in~\cite{bahri2021sharpness}.
Nevertheless, as shown in the \textit{All Avg} column, \texttt{AsyncSAM} consistently achieves higher accuracy than SGD on most GLUE datasets.
Its average accuracy is also comparable to that of the conventional SAM and other SAM variants.
These results demonstrate that, on NLP benchmarks, \texttt{AsyncSAM} achieves model accuracy comparable to the original SAM while benefiting from the system-level efficiency of asynchronous model perturbation.

\subsection{Impact of Model Perturbation Batch Size}

\begin{table}[t]
\small
\setlength{\tabcolsep}{3pt}
\caption{
    Classification performance of \texttt{AsyncSAM} on various heterogeneous systems.
}
\centering
\begin{tabular}{lllccc} \toprule
Benchmark & Perturbation & Update & $\frac{b}{b'}$ & Time (s) & Valid. Acc \\ \midrule
\multirow{5}{*}{CIFAR-10} & RTX A6000 &  RTX A6000 & $1\times$ & $91 \pm 3$ & $92.56 \pm 0.3\%$ \\
& EPYC 7452 & RTX A6000 & $5\times$ & $93\pm 3$ & $92.45 \pm 0.4\%$ \\ \cmidrule{2-6}
& RTX 4060 & RTX 4060 & $1 \times$ & $19 \pm 1$ & $92.60 \pm 0.6\%$\\
& i9-13900HX & RTX 4060 & $3\times$ & $20 \pm 1$ & $92.53 \pm 0.5\%$ \\ 
& i7-12650H & RTX 4060 & $4\times$ & $20 \pm 2$ & $92.47 \pm 0.5\%$ \\ \midrule
\multirow{3}{*}{Oxford} & RTX A6000 & RTX A6000 & $1\times$ & $52 \pm 3$  & $63.86 \pm 1.1\%$ \\
\multirow{3}{*}{Flowers} & EPYC 7452 & RTX A6000 & $5\times$ & $53 \pm 2$ & $63.92 \pm 0.1\%$ \\ \cmidrule{2-6}
\multirow{2.5}{*}{102} & RTX 4060 & RTX 4060 & $1 \times$ & $17 \pm 1$ & $63.51 \pm 0.4\%$ \\
& i9-13900HX & RTX 4060 & $3\times$ & $18 \pm 1$ & $63.63 \pm 0.4\%$ \\ 
& i7-12650H & RTX 4060 & $4\times$ & $18 \pm 1$ & $63.33 \pm 0.3\%$\\ \bottomrule
\end{tabular}

\label{tab:timing}
\end{table}

We analyze the timing and accuracy of the proposed method on heterogeneous systems.
Table~\ref{tab:timing} reports the average epoch time and validation accuracy.
The column $\frac{b}{b'}$ denotes the ratio of the gradient descent to gradient ascent batch size, estimated in advance from the average iteration time.
The measured iteration times on CPUs and GPUs are not substantially different\footnote{The CPU process utilizes all available cores with a highly-optimized Intel MKL library.}, so $b'$ can be set close to $b$ without introducing a system bottleneck, and need not be reduced excessively in general.
Under this setting, \texttt{AsyncSAM} consistently outperforms SGD regardless of the in-node compute gap (see Table~\ref{tab:classification} for SGD accuracies).
For instance, on CIFAR-10, gradient ascent on an AMD EPYC 7452 CPU is roughly $5\times$ slower than gradient descent on an NVIDIA A6000 GPU. With $b'=26$ and $b=128$, \texttt{AsyncSAM} still reaches $92.45 \pm 0.4\%$ accuracy, compared to $91.96 \pm 0.3\%$ for SGD.
For Oxford\_Flowers102, we set $b=40$ and $b'=14$ (Intel i9) or $10$ (Intel i7), depending on the performance gap.
Accuracy remains stable for $\frac{b}{b'} \leq 3\times$ and drops slightly beyond that, but even at $5\times$, \texttt{AsyncSAM} ($63.92 \pm 0.1\%$) still outperforms SGD ($61.37 \pm 2.2\%$).
These results demonstrate that \texttt{AsyncSAM} effectively exploits in-node heterogeneous resources to improve accuracy.




\begin{table}[t]
\caption{
    The classification performance with varying $b'$ settings.
}
\label{tab:b'1}
\centering
\begin{tabular}{lcrc} \toprule
Benchmark & $b$ & $b'$ & Validation Accuracy \\ \midrule
\multirow{4}{*}{CIFAR-10 (ResNet20)} & \multirow{4}{*}{128} & 16 & $92.35 \pm 0.4\%$ \\
& & 32 & $92.47 \pm 0.5\%$ \\
& & 64 & $92.42 \pm 0.7\%$ \\
& & 128 & $92.66 \pm 0.2\%$ \\ \midrule
\multirow{4}{*}{Oxford\_Flowers102} & \multirow{4}{*}{40} & 5 & $63.04 \pm 0.4\%$ \\
& & 10 & $63.28 \pm 0.7\%$ \\
& & 20 & $63.63 \pm 0.1\%$ \\
& & 40 & $63.86 \pm 1.1\%$ \\ \bottomrule
\end{tabular}

\end{table}

To better understand the impact of the model perturbation with the system-aware batch size, we performed an ablation study with CIFAR-10 and Oxford\_Flowers102 benchmarks.
We fix the batch size for the model update and vary that for the model perturbation.
The $b'$ is set to one of $12.5\%$, $25\%$, $50\%$, and $100\%$ of the $b$.
The ablation results are shown in Table \ref{tab:b'1}.
Overall, we do not see a significant accuracy drop as $b'$ decreases in all the experiments.
E.g., CIFAR-10 accuracy drops just $0.31\%$ when $b'$ is set to $12.5\%$ of $b$.
These results ensure that one can safely adjust the batch size for the model perturbation based only on the system capacity so that the model perturbation time is the same as the model update time.

As compared to CIFAR-10, the accuracy drop of Oxford\_Flowers102 is a bit more noticeable.
Since the best-tuned batch size for the model update, $b$, is already quite small, the $b'$ becomes too small.
The difference is mainly due to the extremely small $b'$ settings.
If $b'$ is too small, the gradient variance will be large making the gradient ascent more inaccurate.
Nevertheless, our proposed method with an extremely small $b'$ such as $5$ or $10$ still achieves comparable accuracy to the original SAM which is much higher than that of SGD.
As shown in Table~\ref{tab:b'1}, one can easily find the maximum gradient ascent batch size that makes the gradient ascent time be hidden behind the gradient descent time while maintaining the model accuracy.
Ni et al. proposed K-SAM, a similar approach that adjusts the batch size for the gradient ascent \cite{ni2022k}.
However, K-SAM selects a subset of each mini-batch with the top training loss and run the gradient ascent only with them.
This approach makes their gradient ascent biased toward a certain subset of training samples adversely affecting the convergence property.
In contrast, our proposed method leaves the gradient ascent unbiased and only adjusts the degree of stochasticity in a system-aware manner.
Therefore, by finding the maximum gradient ascent batch size that matches the gradient ascent time to the gradient descent time, \texttt{AsyncSAM} is expected to achieve the stable convergence properties along with the high system efficiency.

\subsection{Visualized Loss Landscape Comparison}

\begin{figure}[t]
\centering
\includegraphics[width=1\columnwidth]{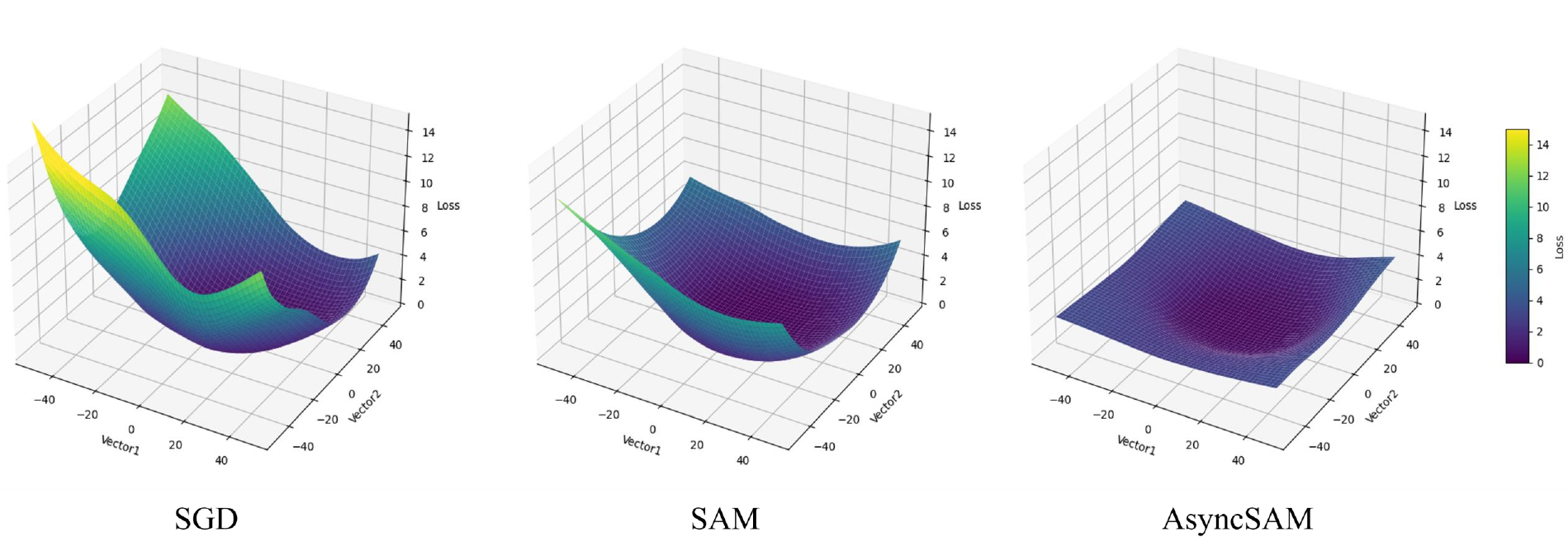}
\caption{
    Loss landscape comparison on CIFAR-10 (ResNet20). 
    It is evident that \texttt{AsyncSAM} guides the model toward a flat region, comparable to that of the original SAM. 
    More loss landscape figures are provided in the Appendix.
}
\label{fig:landscape}
\end{figure}

We visualize the loss landscapes of different SAM variants using the method of~\cite{li2018visualizing}.
The loss is computed over $30 \times 30$ sampled points in the parameter space to construct the landscape.
Figure~\ref{fig:landscape} shows the resulting comparison.
Both SAM and \texttt{AsyncSAM} converge to flatter regions of the loss landscape than SGD, with corresponding validation accuracies of $92.52\%$ and $92.75\%$, versus $92.01\%$ for SGD.
This is consistent with the widely observed link between flat minima and generalization, suggesting that \texttt{AsyncSAM}, like the original SAM, guides the model toward flatter regions of the parameter space.

%% file: section/5_conclusion.tex
\section{Conclusion}

Our study demonstrates that SAM’s perturbation cost can be effectively hidden by introducing limited asynchrony in gradient ascent and appropriately adjusting the batch size based on the available system capability.
Using slightly stale parameters still yields sufficiently accurate perturbation gradients, preserving SAM’s gradient-norm regularization.
This enables efficient utilization of heterogeneous resources such as CPUs and GPUs, accelerating training while maintaining strong generalization.
We believe the proposed asynchronous perturbation scheme will make SAM significantly more practical for real-world applications on modern GPU servers.

\noindent\textbf{Limitations and Future Work} --
This study focuses only on full-model perturbation, ignoring the potential of partial model perturbation approach. Harmonizing asynchronous perturbation and partial model SAM is a promising direction of future work.

%% file: section/appendix.tex
\appendix
\onecolumn
\section{Appendix}

The appendix is structured as follows:
\begin{itemize}
\item Section A.1 provides detailed hyperparameter settings used in our empirical study.
\item Section A.2 provides detailed proofs of the proposed lemmas and theorem.
\item Section A.3 shows additional loss landscape analyzes collected from CIFAR-100 and Tiny-ImageNet.
\end{itemize}

\begin{table}[ht]
\centering
\caption{The hyper-parameter settings.}
\label{tab:setting1}
\begin{tabular}{llrrr} \toprule
Dataset & Model & Batch Size & Init. Learning Rate & Number of Epochs \\ \midrule
CIFAR-10 & ResNet20 & 128 & 0.1 & 150 \\
CIFAR-100 & Wide-ResNet-28 & 128 & 0.1 & 200 \\
Oxford\_Flowers102 & Wide-ResNet-16 & 40 & 0.1 & 100 \\
Google Speech & CNN & 128 & 0.1 & 10 \\
CIFAR-100 & ViT-b16 & 40 & 0.01 & 20 \\
Tiny-ImageNet & ResNet50 & 256 & 0.1 & 200 \\
ImageNet & ResNet50 & 256 & 0.1 & 90 \\
\bottomrule
\end{tabular}

\end{table}

\begin{table}[ht]
\centering
\caption{The optimizer-specific hyper-parameter settings.}
\label{tab:setting2}
\begin{tabular}{ll} \toprule
Optimizer & Settings \\ \midrule
SGD & momentum = $0.9$ \\
SAM & $r=0.1$ \\
Generalized SAM & $r=0.1$, $\alpha = 0.7 \sim 0.9$ \\
ESAM & $r = 0.1$, $\beta = 0.6$, $\gamma = 0.6 \sim 1$ \\
LookSAM & $r=0.1$, gradient ascent reuse interval = 2 \\
MESA & $\beta = 0.995$, $\lambda = 0.8$, $\tau = 1.5$, start epoch = 5\\ 
AE-SAM & $r=0.1$, $\lambda_1 = -1$, $\lambda_2=1$, $\epsilon=0.9$\\
Asynchronous SAM & $r=0.05 \sim 0.1$, $\tau=1$, $\frac{b}{b'} = \{12.5\%, 25\%, 50\%, 100\% \}$\\
\bottomrule
\end{tabular}

\end{table}

\begin{table}[ht]
\centering
\caption{NLP experimental settings. The same $r$ and $\alpha$ settings are used for the conventional SAM and our asynchronous SAM.}
\label{tab:setting3}

\begin{tabular}{lcccccccc} \toprule
Dataset & CoLA & SST-2 & MRPC & STS-B & QQP & MNLI & QNLI & RTE \\ \midrule
Batch Size & 128 & 128 & 128 & 128 & 128 & 128 & 128 & 128 \\
Learning Rate & 0.001 & 0.001 & 0.001 & 0.001 & 0.001 & 0.001 & 0.001 & 0.001 \\
Number of Epochs & 10 & 10 & 10 & 10 & 10 & 10 & 10 & 10 \\
$r$ & 0.01 & 0.01 & 0.01 & 0.01 & 0.01 & 0.01 & 0.01 & 0.01 \\
$\alpha$ & 0.6 & 0.6 & 0.6 & 0.7 & 0.7 & 0.7 & 0.7 & 0.9 \\
\bottomrule
\end{tabular}

\end{table}

\subsection { Hyperparameter settings} \label{appendix:settings}

In this section, we provide our best-tuned hyperparameter settings used in all our experiments.
The hyperparameters settings used in image classification benchmarks are shown in Table~\ref{tab:setting1}.
The optimizer-specific settings are shown in Table~\ref{tab:setting2}.
Finally, the NLP experimental settings are shown in Table~\ref{tab:setting3}.
All hyperparameters are tuned using grid search with parameter-specific grid sizes.
For instance, the learning rate is tuned with a step size of 0.1 for vision tasks and 0.0005 for NLP tasks.
The batch size is chosen from {64, 128, 256}.
All reported results are averaged over at least two independent runs.
We believe the experiments conducted in this study can be all precisely reproduced using our hyperparameter settings.
If a setting is written like $x \sim y$, we choose the best one between $x$ and $y$ depending on the benchmarks.

\subsection { Theoretical Analysis} \label{appendix:theory}
To make this appendix self-contained, we present the assumptions used in our analysis as follows:
\\
\textbf{Assumption 1.} (Smoothness): There exists a constant $\beta \geq 0$ such that $\| \nabla l(u) - \nabla l(v) \| \leq \beta \| u - v \| \mathrm{ , } \forall u, v \in \mathbb{R}^d$.
\\
\textbf{Assumption 2.} (Unbiased Gradient): The gradient estimator is unbiased such that $\mathbb{E}_{\xi \sim D}\left[ \nabla l(w, \xi) \right] = \nabla L(w)$.
\\
\textbf{Assumption 3.} (Bounded Variance): There exists a constant $\sigma \geq 0$ such that $\mathbb{E}_{\xi \sim D}[ \| \nabla l(w, \xi) - \nabla L(w) \|^2] \leq \sigma^2$.

Then, we present a couple of useful lemmas here.
Note that our analysis borrows the proof structure used in \cite{andriushchenko2022towards}.

\begin{lemma} \label{lemma:async_2}
Given a $\beta$-smooth loss function $L(x)$, we have the following bound for any $r > 0$ and $x \in \mathbb{R}^d$.
\begin{align*}
\left\langle \nabla L(w_t + r \nabla L(w_{t-\tau})), \nabla L(w_t) \right\rangle
\geq \left( 1- \frac{\beta r}{4} \right) \| \nabla L(w_t) \|^2 - r\beta \| \nabla L(w_{t-\tau}) \|^2
\end{align*}
\end{lemma}

\begin{proof}
\begin{align}
\langle \nabla L(w_t + r \nabla L(w_{t-\tau})), \nabla L(w_t) \rangle &= \left\langle 
    \nabla L(w_t + r \nabla L(w_{t-\tau})) 
    - \nabla L(w_t), \nabla L(w_t) 
\right\rangle+ \| \nabla L(w_t) \|^2 \nonumber \\
&\geq -\left\| \nabla L(w_t + r \nabla L(w_{t-\tau})) 
    - \nabla L(w_t)\right\| \left\| \nabla L(w_t) \right\| + \| \nabla L(w_t) \|^2 \label{eq:CS} \\
&\geq -\beta \left\| r \nabla L(w_{t-\tau}) \right\| \left\| \nabla L(w_t) \right\| + \| \nabla L(w_t) \|^2 \label{eq:beta} \\ 
&= -\beta r \left\| \nabla L(w_{t-\tau}) \right\| \left\| \nabla L(w_t) \right\| + \| \nabla L(w_t) \|^2 \nonumber \\
&\geq -\beta r \left\| \nabla L(w_{t-\tau}) \right\|^2 - \frac{\beta r}{4} \left\| \nabla L(w_t) \right\|^2 + \| \nabla L(w_t) \|^2 \label{eq:CM-AM} \\
&= -\beta r \left\| \nabla L(w_{t-\tau}) \right\|^2 + \left(1 -\frac{\beta r}{4} \right) \left\| \nabla L(w_t) \right\|^2 \nonumber
\end{align}
where (\ref{eq:CS}) holds based on Cauchy-Schwarz inequality, (\ref{eq:beta}) is based on $\beta$-smoothness, and (\ref{eq:CM-AM}) holds due to the fact $ab \leq \lambda a^2 /2 + b^2/ \lambda 2$, where $\lambda=2$.
\end{proof}


\begin{lemma}
\label{lemma:async_dot}
We consider the classical SAM which uses the same mini-batch when calculating the gradient ascent and the gradient descent.
Then, given a $\beta$-smooth loss function $L(x)$, we have the following bound for any $r > 0$ and $x \in \mathbb{R}^d$.
\begin{align}
    &\mathbb{E} \left[ \langle \nabla L_{t+1}(w_t + r \nabla L_{t-\tau+1}(w_{t-\tau})), \nabla L(w_t) \rangle \right]\\ 
    &\geq \left( \frac{1}{2} - \frac{r\beta}{4} \right) \mathbb{E}\left[ \| \nabla L(w_t) \|^2 \right]  - r\beta \mathbb{E}\left[ \| \nabla L(w_{t-\tau}) \|^2 \right] - \frac{\beta^2 r^2 \sigma^2}{2b}. \nonumber
\end{align}
\end{lemma}
\begin{proof}
We first define the asynchronous gradient ascent step $\Tilde{w}_t = w_t + r\nabla L(w_{t-\tau})$.

\begin{align*}
&\mathbb{E} \left[
    \left\langle 
        \nabla L_{t+1}(w_t + r \nabla L_{t-\tau+1}(w_{t-\tau})),\,
        \nabla L(w_t) 
    \right\rangle 
\right]\\
&= \mathbb{E} \left[
    \left\langle 
        \nabla L(w_t + r \nabla L_{t-\tau+1}(w_{t-\tau})),\,
        \nabla L(w_t) 
    \right\rangle 
\right] \\
&= \mathbb{E} \left[
    \left\langle 
        \nabla L(w_t + r \nabla L_{t-\tau+1}(w_{t-\tau})) 
        - \nabla L(\Tilde{w}_t)
        + \nabla L(\Tilde{w}_t),\,
        \nabla L(w_t) 
    \right\rangle 
\right] \\
&= \underset{E_1}{\underbrace{
    \mathbb{E} \left[
        \left\langle 
            \nabla L(w_t + r \nabla L_{t-\tau+1}(w_{t-\tau})) 
            - \nabla L(\Tilde{w}_t),\,
            \nabla L(w_t) 
        \right\rangle 
    \right]
}}
+ \underset{E_2}{\underbrace{
    \mathbb{E} \left[
        \left\langle 
            \nabla L(\Tilde{w}_t),\,
            \nabla L(w_t) 
        \right\rangle 
    \right]
}}.
\end{align*}
Then, we will bound $E_1$ and $E_2$ separately.
First, $E_1$ is lower-bounded as follows.
\begin{align}
E_1 &= \mathbb{E} \left[ 
    \left\langle 
        \nabla L(w_t + r \nabla L_{t-\tau+1}(w_{t-\tau})) 
        - \nabla L(\Tilde{w}_t), 
        \nabla L(w_t) 
    \right\rangle 
\right] \notag \\
&\geq 
-\frac{1}{2} \mathbb{E} \left[
    \left\| 
        \nabla L(w_t + r \nabla L_{t-\tau+1}(w_{t-\tau})) 
        - \nabla L(\Tilde{w}_t) 
    \right\|^2 
\right]
- \frac{1}{2} \mathbb{E} \left[ \| \nabla L(w_t) \|^2 \right] \notag \\
&\geq 
-\frac{\beta^2}{2} \mathbb{E} \left[ 
    \left\| 
        w_t + r \nabla L_{t-\tau+1}(w_{t-\tau}) 
        - \Tilde{w}_t 
    \right\|^2 
\right]
- \frac{1}{2} \mathbb{E} \left[ \| \nabla L(w_t) \|^2 \right]
\label{eq:async_beta} \\
&=
-\frac{\beta^2}{2} \mathbb{E} \left[ 
    \left\| 
        r \nabla L_{t-\tau+1}(w_{t-\tau}) 
        - r \nabla L(w_{t-\tau}) 
    \right\|^2 
\right]
- \frac{1}{2} \mathbb{E} \left[ \| \nabla L(w_t) \|^2 \right] \notag \\
&\geq 
-\frac{\beta^2 r^2 \sigma^2}{2b'}
- \frac{1}{2} \mathbb{E} \left[ \| \nabla L(w_t) \|^2 \right]
\label{eq:async_sigma}
\end{align}

where (\ref{eq:async_beta}) is based on the smoothness assumption. The final equality, (\ref{eq:async_sigma}), is based on the bounded variance assumption.
Then, $E_2$ is lower-bounded directly based on Lemma \ref{lemma:async_2} as follows.
\begin{align*}
    E_2 &= \mathbb{E} \left[ 
        \langle \nabla L(\Tilde{w}_t), \nabla L(w_t) \rangle 
    \right]
    \geq \left(1 - \frac{r\beta}{4} \right) \mathbb{E}[\| \nabla L(w_t) \|^2] 
        - r\beta \mathbb{E}[ \| \nabla L(w_{t-\tau}) \|^2].
\end{align*}
Summing up these bounds, we have
\begin{align}
    \mathbb{E} \left[ \langle \nabla L_{t+1}(w_t + r \nabla L_{t-\tau+1}(w_{t-\tau})), \nabla L(w_t) \rangle \right] &= E_1 + E_2 \nonumber \\
    &\geq \left( \frac{1}{2} - \frac{r\beta}{4} \right) \mathbb{E}\left[ \| \nabla L(w_t) \|^2 \right] - r\beta \mathbb{E}[ \| \nabla L(w_{t-\tau}) \|^2] - \frac{\beta^2 r^2 \sigma^2}{2b'} \nonumber
\end{align}
\end{proof}

\begin{lemma}
\label{lemma:async_frame_appendix}
Under the assumption of $\beta$ smoothness and the bounded variance, our asynchronous SAM guarantees the following if $\eta \leq \frac{1}{2\beta}$ and $r \leq \frac{1}{\beta}$:
\begin{align}
\mathbb{E}[L(w_{t+1})] &\leq \mathbb{E}[L(w_t)] -\frac{\eta}{4} \mathbb{E}\left[ \left\| \nabla L(w_t) \right\|^2 \right] + \eta r\beta \mathbb{E}\left[ \left\| \nabla L(w_{t-\tau}) \right\|^2 \right] \nonumber \\
&\qquad + \frac{\eta^2 \beta \sigma^2}{b} + \left(2\eta^2 r \beta^2 + \frac{\eta r \beta}{2} \right) \frac{\sigma^2}{b'}. \nonumber
\end{align}
\end{lemma}
\begin{proof}
We first define the perturbed model parameter at round $t$ as $w_{t+1/2}$.

Then, from the smoothness assumption, we begin with the following condition:
\begin{align}
L(w_{t+1}) &\leq L(w_t) - \eta \langle \nabla L_{t+1}(w_{t+1/2}), \nabla L(w_t) \rangle \nonumber + \frac{\eta^2 \beta}{2} \left\| \nabla L_{t+1} (w_{t+1/2}) \right\|^2. \nonumber
\end{align}
By taking expectation and based on the variance assumption, we have
\begin{align}
\mathbb{E}[L(w_{t+1})] &\leq \mathbb{E}[L(w_t)] - \eta \mathbb{E}\left[\langle \nabla L(w_{t+1/2}), \nabla L(w_t) \rangle \right] \nonumber + \frac{\eta^2 \beta}{2} \mathbb{E}\left[ \left\| \nabla L_{t+1}(w_{t+1/2}) \right\|^2 \right] \nonumber \\
&\leq \mathbb{E}[L(w_t)] - \eta \mathbb{E}\left[\langle \nabla L(w_{t+1/2}), \nabla L(w_t) \rangle \right] \nonumber \\  
&\quad + \eta^2 \beta \mathbb{E}\left[ \left\| \nabla L_{t+1}(w_{t+1/2}) - \nabla L(w_{t+1/2}) \right\|^2 \right] \nonumber  + \eta^2 \beta \mathbb{E}\left[ \left\| \nabla L(w_{t+1/2}) \right\|^2 \right] \nonumber \\
&\leq \mathbb{E}[L(w_t)] - \eta \mathbb{E}\left[\langle \nabla L(w_{t+1/2}), \nabla L(w_t) \rangle \right]+ \frac{\eta^2 \beta \sigma^2}{b} \nonumber \\ 
&\quad  + \eta^2 \beta \mathbb{E}\left[ \left\| \nabla L(w_{t+1/2}) \right\|^2 \right] \label{eq:frame1} \\
&= \mathbb{E}[L(w_t)] - \eta \mathbb{E}\left[\langle \nabla L(w_{t+1/2}), \nabla L(w_t) \rangle \right] \nonumber + \frac{\eta^2 \beta \sigma^2}{b}\nonumber \\
& \quad + \eta^2 \beta \mathbb{E}\left[ \left\| \nabla L(w_{t+1/2}) - \nabla L(w_t) \right\|^2 \right] \nonumber -\eta^2 \beta \mathbb{E}\left[ \left\| \nabla L(w_t) \right\|^2 \right] \nonumber \\
& \quad + 2\eta^2 \beta \mathbb{E}\left[ \langle \nabla L(w_{t+1/2}), \nabla L(w_t) \rangle \right] \nonumber \\
&= \mathbb{E}[L(w_t)] - \eta^2 \beta \mathbb{E}\left[ \left\| \nabla L(w_t) \right\|^2 \right] + \frac{\eta^2 \beta \sigma^2}{b} \nonumber \\
&\quad + \eta^2 \beta \mathbb{E}\left[ \left\| \nabla L(w_{t+1/2}) - \nabla L(w_t) \right\|^2 \right] \nonumber \\
&\quad -\eta(1 - 2\eta\beta) \mathbb{E}\left[ \langle \nabla L(w_{t+1/2}), \nabla L(w_t) \rangle \right]. \label{eq:frame2}
\end{align}
Now, based on the smoothness assumption, (\ref{eq:frame2}) can be re-written as follows.

\begin{align}
\mathbb{E}[L(w_{t+1})] &\leq \mathbb{E}[L(w_t)] - \eta^2 \beta \mathbb{E}\left[ \left\| \nabla L(w_t) \right\|^2 \right] + \frac{\eta^2 \beta \sigma^2}{b} \nonumber + \eta^2 \beta^3 r^2 \mathbb{E}\left[ \left\| w_{t+1/2} - w_t \right\|^2 \right] \nonumber \\
&\qquad -\eta(1 - 2\eta\beta) \mathbb{E}\left[ \langle \nabla L(w_{t+1/2}), \nabla L(w_t) \rangle \right] \nonumber \\
&= \mathbb{E}[L(w_t)] - \eta^2 \beta \mathbb{E}\left[ \left\| \nabla L(w_t) \right\|^2 \right] + \frac{\eta^2 \beta \sigma^2}{b} \nonumber \\
&\qquad + \eta^2 \beta^3 \mathbb{E}\left[ \left\| w_t + r\nabla L_{t-\tau+1}(w_{t-\tau}) - w_t \right\|^2 \right] \nonumber \\
&\qquad -\eta(1 - 2\eta\beta) \mathbb{E}\left[ \langle \nabla L(w_{t+1/2}), \nabla L(w_t) \rangle \right] \nonumber \\
&= \mathbb{E}[L(w_t)] - \eta^2 \beta \mathbb{E}\left[ \left\| \nabla L(w_t) \right\|^2 \right] + \frac{\eta^2 \beta \sigma^2}{b} \nonumber \\
&\qquad + \eta^2 \beta^3 r^2 \mathbb{E}\left[ \left\| \nabla L_{t-\tau + 1}(w_{t-\tau}) \right\|^2 \right] \nonumber \\
&\qquad -\eta(1 - 2\eta\beta) \mathbb{E}\left[ \langle \nabla L(w_{t+1/2}), \nabla L(w_t) \rangle \right] \nonumber \\
&= \mathbb{E}[L(w_t)] + \frac{\eta^2 \beta \sigma^2}{b} + \frac{2\eta^2 \beta^3 r^2 \sigma^2}{b'} \nonumber \\
&\qquad - \eta^2 \beta \mathbb{E}\left[ \left\| \nabla L(w_t) \right\|^2 \right] + 2 \eta^2 \beta^3 r^2 \mathbb{E}\left[ \left\| \nabla L(w_{t-\tau}) \right\|^2 \right] \nonumber \\
&\qquad -\eta(1 - 2\eta\beta) \mathbb{E}\left[ \langle \nabla L(w_{t+1/2}), \nabla L(w_t) \rangle \right]. \nonumber
\end{align}

Then, if $\eta \leq \frac{1}{2\beta}$, based on Lemma~\ref{lemma:async_dot}, the final term on the right-hand side can be replaced and rearranged as follows:
\begin{align}
\mathbb{E}[L(w_{t+1})] &\leq \mathbb{E}[L(w_t)] + \frac{\eta^2 \beta \sigma^2}{b} + \frac{2\eta^2 \beta^3 r^2 \sigma^2}{b'} \nonumber \\
&\qquad - \eta^2 \beta \mathbb{E}\left[ \left\| \nabla L(w_t) \right\|^2 \right] + 2 \eta^2 \beta^3 r^2 \mathbb{E}\left[ \left\| \nabla L(w_{t-\tau}) \right\|^2 \right] \nonumber \\
&\qquad -\eta (1-2\eta\beta) \left(\frac{1}{2} - \frac{r\beta}{4} \right) \mathbb{E}\left[ \left\| \nabla L(w_t) \right\|^2 \right] \nonumber \\
&\qquad +\eta (1-2\eta\beta)r\beta \mathbb{E}\left[ \left\| \nabla L(w_{t-\tau}) \right\|^2 \right] \nonumber \\
&\qquad + \eta (1-2\eta\beta) \frac{\beta^2 r^2 \sigma^2}{2b'} \nonumber \\
&= \mathbb{E}[L(w_t)] \nonumber \\
&\qquad -\eta\left( \frac{1}{2}-\frac{r\beta}{4} + \frac{\eta r \beta^2}{2} \right) \mathbb{E}\left[ \left\| \nabla L(w_t) \right\|^2 \right] \nonumber \\
&\qquad +\eta r\beta \left(1-2\eta\beta \left( 1-r\beta \right) \right) \mathbb{E}\left[ \left\| \nabla L(w_{t-\tau}) \right\|^2 \right] \nonumber \\
&\qquad + \frac{\eta^2 \beta \sigma^2}{b} + \left(2\eta^2 \beta^3 r^2  + \frac{\eta\beta^2 r^2 (1-2\eta\beta)}{2} \right)\cdot \frac{\sigma^2}{b'} \nonumber
\end{align}

If $r \leq \frac{1}{\beta}$, the coefficient of $\mathbb{E}[\| \nabla L(w_t) \|^2]$ term can be simplified as follows:
\begin{align}
\eta\left(\frac{1}{2} -\frac{r\beta}{4} + \frac{\eta r\beta^2}{2} \right) &\geq \eta\left(\frac{1}{2} -\frac{r\beta}{4} \right) \geq \frac{\eta}{4}.\nonumber
\end{align}

Under the same condition, the coefficient of $\mathbb{E}[\| \nabla L(w_{t-\tau}) \|^2]$ term can be also simplified as follows:
\begin{align}
\eta r\beta \left(1-2\eta\beta \left( 1-r\beta \right) \right) &\leq \eta r\beta.\nonumber
\end{align}

Since $\eta, \beta > 0$, $(1-2\eta\beta) < 1$, and the coefficients of the noise terms can also be simplified under the same condition $r\leq \frac{1}{\beta}$:
\begin{align}
\frac{\eta \beta^2 r^2 (1-2\eta\beta)}{2} &\leq \frac{\eta r^2 \beta^2}{2} \leq \frac{\eta r\beta}{2} \nonumber \\
2\eta^2 \beta^3 r^2 &\leq 2\eta^2 r\beta^2. \nonumber
\end{align}

Finally, plugging in the above coefficients, we have:
\begin{align}
\mathbb{E}[L(w_{t+1})] &\leq \mathbb{E}[L(w_t)] -\frac{\eta}{4} \mathbb{E}\left[ \left\| \nabla L(w_t) \right\|^2 \right] + \eta r\beta \mathbb{E}\left[ \left\| \nabla L(w_{t-\tau}) \right\|^2 \right] \nonumber \\
&\qquad + \frac{\eta^2 \beta \sigma^2}{b} + \left(2\eta^2 r \beta^2 + \frac{\eta r \beta}{2} \right) \frac{\sigma^2}{b'}. \nonumber
\end{align}
\end{proof}

\begin{theorem}
Under assumption 1--3, if $\eta \leq \frac{1}{2\beta}$ and $r < \frac{1}{4\beta}$, asynchronous SAM satisfies:
\begin{align}
\frac{1}{T} \sum_{t=0}^{T-1} \mathbb{E} \left[\left\| \nabla L(w_t) \right\|^2 \right] &\leq \frac{4}{\eta(1-4r\beta)T} \left( \mathbb{E}\left[L(w_0)\right] - \mathbb{E}\left[L(w_T) \right] \right) + \frac{4r\beta \tau}{(1-4r\beta) T} \mathbb{E} \left[ \left\| \nabla L(w_0) \right\|^2 \right] \nonumber \\
&\qquad + \frac{4\eta \beta}{1-4r \beta} \cdot \frac{\sigma^2}{b} + \frac{2r\beta (1+4\eta\beta)}{1-4r\beta} \cdot \frac{\sigma^2}{b'}.
\end{align}
\end{theorem}

\begin{proof}
First, for the sake of simplicity, we define $c_1=\frac{\eta}{4}$, $c_2=\eta r \beta$, and $c_3 = \frac{\eta^2 \beta \sigma^2}{b} + \left(2\eta^2 r \beta^2 + \frac{\eta r \beta}{2} \right) \frac{\sigma^2}{b'}$.
Then, based on Lemma~\ref{lemma:async_frame_appendix}, we begin with:
\begin{align}
\mathbb{E}\left[L(w_{t+1})\right] - \mathbb{E}\left[L(w_t)\right] &\leq -c_1 \mathbb{E}\left[ \left\| \nabla L(w_t) \right\|^2 \right] + c_2 \mathbb{E}\left[ \left\| \nabla L(w_{t-\tau}) \right\|^2 \right] + c_3. \nonumber
\end{align}

By summing up across $T$ steps, we have a telescoping sum as follows:
\begin{align}
\mathbb{E}\left[L(w_{T})\right] - \mathbb{E}\left[L(w_0)\right] &\leq -c_1 \sum_{t=0}^{T-1} \mathbb{E}\left[ \left\| \nabla L(w_t) \right\|^2 \right] + c_2 \sum_{t=0}^{T-1} \mathbb{E}\left[ \left\| \nabla L(w_{t-\tau}) \right\|^2 \right] + c_3 T. \label{eq:tonly}
\end{align}

By replacing $t-\tau$ with $s$, the second term on the right-hand side can be re-written as follows:
\begin{align}
    \sum_{t=0}^{T-1} \mathbb{E}\left[\left\| \nabla L(w_{t-\tau}) \right\|^2\right] &= \sum_{s=-\tau}^{T-1-\tau} \mathbb{E}\left[\left\| \nabla L(w_{s}) \right\|^2\right] \nonumber \\
    &\leq \sum_{t=0}^{T-1} \mathbb{E}\left[\left\| \nabla L(w_{t}) \right\|^2\right] + \tau \mathbb{E} \left[ \left\| \nabla L(w_0) \right\|^2 \right]. \label{eq:stot}
\end{align}

By plugging in (\ref{eq:stot}) to (\ref{eq:tonly}), we have:
\begin{align}
\mathbb{E}\left[L(w_{T})\right] - \mathbb{E}\left[L(w_0)\right] &\leq - (c_1 - c_2) \sum_{t=0}^{T-1} \mathbb{E}\left[ \left\| \nabla L(w_t) \right\|^2 \right] + c_2 \tau \mathbb{E} \left[ \left\| \nabla L(w_0) \right\|^2 \right] + c_3 T.
\end{align}

By rearranging the terms and dividing both sides by $T(c_1 - c_2)$:
\begin{align}
    \frac{1}{T} \sum_{t=0}^{T-1} \mathbb{E} \left[\left\| \nabla L(w_t) \right\|^2 \right] &\leq \frac{1}{T(c_1-c_2)} \left( \mathbb{E}\left[L(w_0)\right] - \mathbb{E}\left[L(w_T) \right] \right) + \frac{c_2 \tau}{T(c_1-c_2)} \mathbb{E} \left[ \left\| \nabla L(w_0) \right\|^2 \right] + \frac{c_3}{c_1-c_2}.
\end{align}

Finally, if $r\beta < \frac{1}{4}$, we have:
\begin{align}
\frac{1}{T} \sum_{t=0}^{T-1} \mathbb{E} \left[\left\| \nabla L(w_t) \right\|^2 \right] &\leq \frac{4}{\eta(1-4r\beta)T} \left( \mathbb{E}\left[L(w_0)\right] - \mathbb{E}\left[L(w_T) \right] \right) + \frac{4r\beta \tau}{(1-4r\beta) T} \mathbb{E} \left[ \left\| \nabla L(w_0) \right\|^2 \right] \nonumber \\
&\qquad + \frac{4\eta \beta}{1-4r \beta} \cdot \frac{\sigma^2}{b} + \frac{2r\beta (1+4\eta\beta)}{1-4r\beta} \cdot \frac{\sigma^2}{b'}, \label{eq:thm}
\end{align}
where (\ref{eq:thm}) holds since $c_1 - c_2 = \frac{\eta(1-4 r\beta)}{4}$.
\end{proof}

\clearpage

\subsection {Additional loss landscape analysis}
\label{appendix:landscape}

\begin{figure*}[t]
\centering
\includegraphics[width=0.75\textwidth]{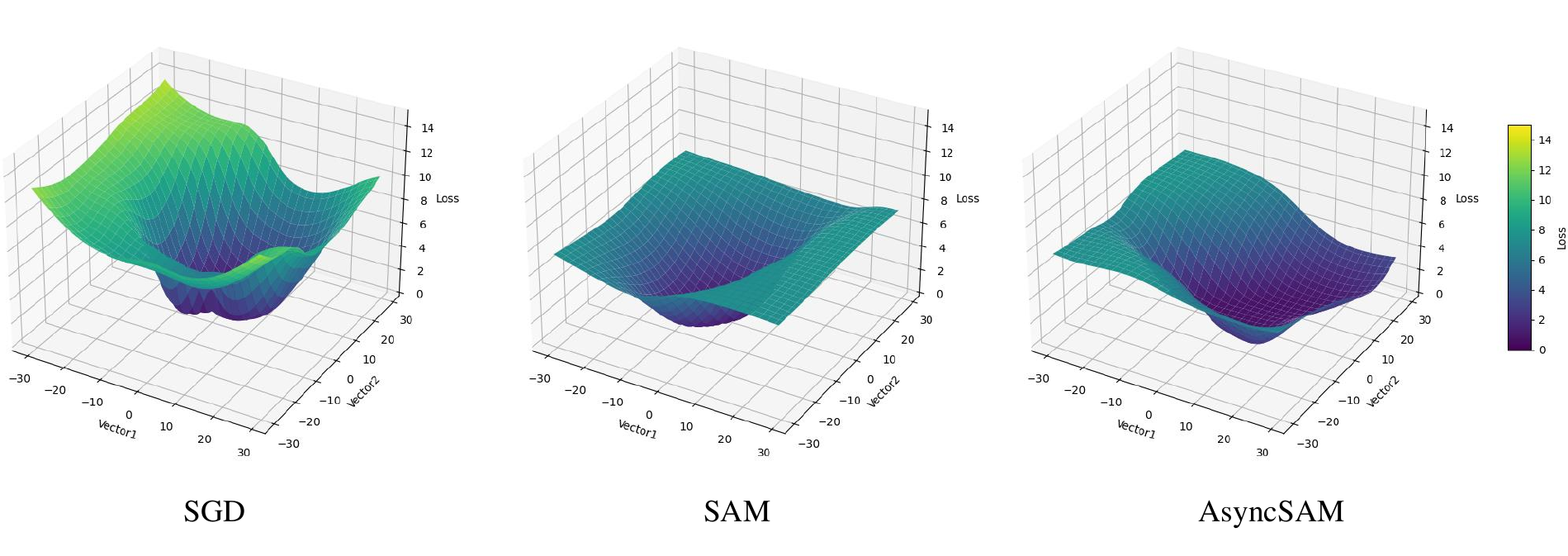}
\caption{
    Loss landscape comparison on Tiny-ImageNet (ResNet50).
}
\label{fig:landscape_tiny}
\end{figure*}

\begin{figure*}[t]
\centering
\includegraphics[width=0.75\textwidth]{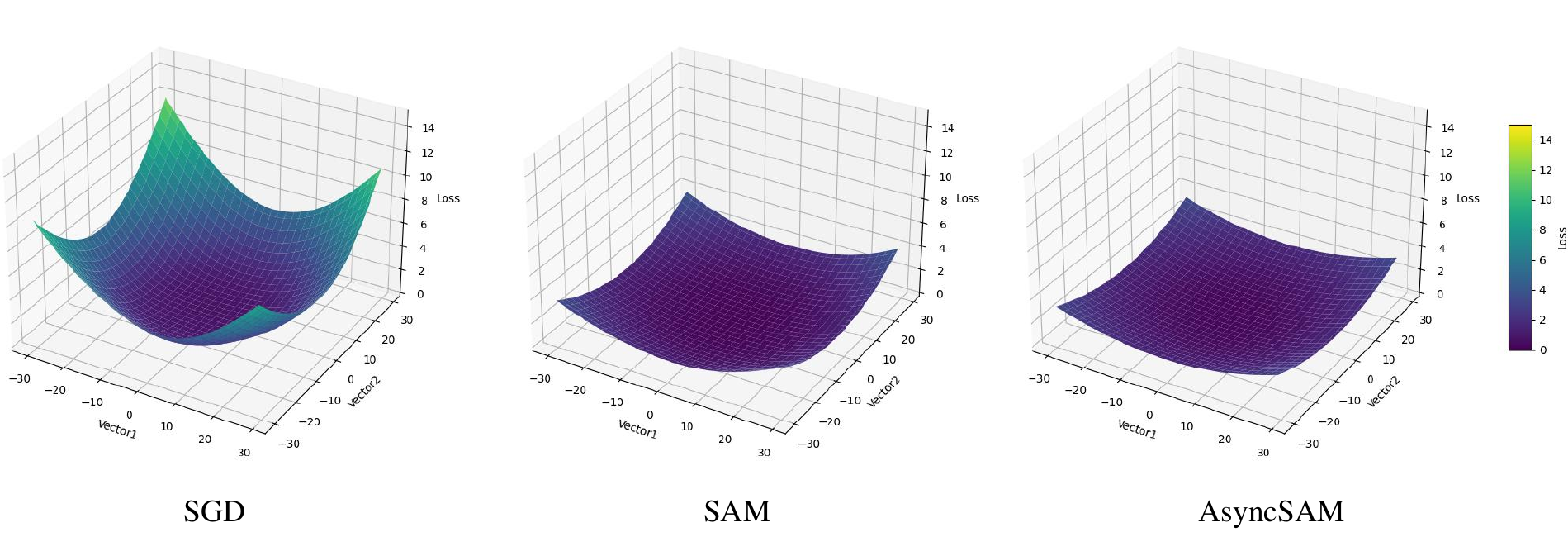}
\caption{
    Loss landscape comparison on CIFAR-100 (ViT-b16).
}
\label{fig:landscape_vit}
\end{figure*}

To better support our arguments, we provide additional loss landscape comparisons on larger benchmark datasets in this appendix.
We first generate two orthogonal random vectors and perturb the model parameters by scaling each vector with a small constant and adding them to the parameters.
Using the resulting loss values, we plot the 3D loss landscapes shown in Figures~\ref{fig:landscape_tiny} and~\ref{fig:landscape_vit}.
To clearly visualize the flatness of the landscapes, we fix the $z$-axis range to 15 for all plots and compare the figures.
The model perturbation is applied to all individual layers, and we report the layer that produces the most distinguishable loss landscapes.
It has been shown that certain layers play a more critical role in determining a model’s generalization performance than others~\cite{lee2024layer}.
Based on our experiments, we found that layer $12$ for CIFAR-10, layer $2$ for Tiny-ImageNet, and layer $13$ for CIFAR-100 fine-tuning showed the clearest trends.

First, at the boundary, the loss of SGD increases much more than that of SAM and \texttt{AsyncSAM}, indicating that SGD drives the model into a sharper region.
In contrast, SAM and \texttt{AsyncSAM} exhibit much flatter loss landscapes.
These comparisons empirically demonstrate that \texttt{AsyncSAM} provides the same gradient-norm regularization effect as the original SAM.